\documentclass[10pt,journal,letterpaper,compsoc]{IEEEtran}

\usepackage{amssymb}
\usepackage{mathtools}
\usepackage{amsthm}
\usepackage{balance}

\usepackage{microtype}
\usepackage{graphicx}
\usepackage{subfigure}
\usepackage{booktabs} % for professional tables
\DeclareUnicodeCharacter{2212}{-}

\usepackage{soul}
\usepackage{url}
\usepackage{bm}
\usepackage[utf8]{inputenc}
\usepackage[small,justification=centering]{caption}
\usepackage{graphicx}
\usepackage{float}

\usepackage{algorithm}
\usepackage{algorithmic}
\usepackage{makecell}
\usepackage{multirow}
\usepackage{txfonts}
\usepackage{enumitem}
\usepackage{threeparttable}
\usepackage[capitalize,noabbrev]{cleveref}

\theoremstyle{plain}
\newtheorem{theorem}{Theorem}[section]
\newtheorem{proposition}[theorem]{Proposition}

\theoremstyle{definition}
\newtheorem{definition}[theorem]{Definition}

\theoremstyle{remark}

\usepackage[textsize=tiny]{todonotes}

\begin{document}
\title{ O-ViT: Orthogonal Vision Transformer}

\author{Yanhong~Fei,
        % Jiepin~Ding,~\IEEEmembership{Member,~IEEE,}
        Yingjie~Liu,
        Xian Wei,
        and~Mingsong~Chen
       \IEEEcompsocitemizethanks{
\IEEEcompsocthanksitem 
Yanhong Fei, Xian Wei and  Mingsong Chen    are with the Shanghai Key
Lab of Trustworthy Computing at East China Normal University,
Shanghai, 200062, China (51205902009@stu.ecnu.edu.cn, xian.wei@tum.de, mschen@sei.ecnu.edu.cn). 
Yingjie Liu is with Fujian Institute of Research on the Structure of Matter, Chinese Academy of Sciences.
}}

%\twocolumn[
%\icmltitle{O-ViT: Orthogonal Vision Transformer}




% \begin{icmlauthorlist}
% \icmlauthor{Yanhong Fei}{ecnu}
% \icmlauthor{Yingjie Liu}{fj}
% \icmlauthor{Xian Wei}{ecnu}
% \icmlauthor{Mingsong Chen}{ecnu}
% \end{icmlauthorlist}

% \icmlcorrespondingauthor{Xian Wei}{xwei@sei.ecnu.edu.cn}
%  \icmlcorrespondingauthor{Mingsong Chen}{mschen@sei.ecnu.edu.cn}

%\icmlkeywords{Machine Learning, ICML}

\vskip 0.3in

%\printAffiliationsAndNotice{}




\IEEEcompsoctitleabstractindextext{%
\begin{abstract}
Inspired by the tremendous success of the self-attention mechanism in natural language processing, the Vision Transformer (ViT) creatively applies it to image patch sequences and achieves incredible performance. However, the scaled dot-product self-attention of ViT brings about scale ambiguity to the structure of the original feature space. To address this problem, we propose a novel method named
Orthogonal Vision Transformer ({O-ViT}), to optimize ViT from the geometric perspective. O-ViT limits parameters of self-attention blocks to be on the norm-keeping orthogonal manifold, which can keep the geometry of the feature space. Moreover, O-ViT achieves both orthogonal constraints and cheap optimization overhead by adopting a surjective mapping between the orthogonal group and its Lie algebra. We have conducted comparative experiments on image recognition tasks to demonstrate O-ViT's validity and experiments show that O-ViT can boost the performance of ViT by up to 3.6\%.
\end{abstract}
}

\maketitle

\section{Introduction}
Recent years have witnessed ViT taking over the Convolution Neural Network (CNN) and achieving dramatic success in computer vision, such as image classification \cite{touvron2020training,yuan2021tokenstotoken}. 
It benefits from transferring the self-attention mechanism \cite{DBLP:conf/nips/VaswaniSPUJGKP17}, originally applied to language sequences, to vision tasks to learn the internal characteristics of image patch sequences \cite{50650}.
Convolution operations gradually expand the view of the CNN kernel layer by layer. By comparison, the self-attention mechanism allows ViT to obtain the global feature even in shallow layers \cite{50650}.
Nonetheless, linear transformations in the self-attention of ViT bring about  \textit{scale ambiguity} to the structure of the feature space. Besides, the softmax function for normalization has the risk of leading to \textit{gradient vanishing} problems \cite{sun2020gradient}. Both restrict ViT to find the optimal solution or slower its optimization.

This motivates us to explore the optimization of ViT on the orthogonal manifold. To achieve this goal, we put forward a novel method named Orthogonal Vision Transformer ({O-ViT}). Each Matrix $A$ that resides on the orthogonal manifold has the following property \cite{Wang2020OrthogonalCN}:
\begin{equation}
\label{orth_definition}
    A^TA = AA^T = E,
\end{equation}
where $E$ is the identity matrix. On the one hand, orthogonal transformations have numerical stability \cite{DBLP:conf/icassp/LahlouO16} and will not enlarge the gap between data. On the other hand, orthogonal transformations will not compress and stretch the original space, which protects the internal information of the original space from being lost. Furthermore, orthogonal optimizations  have fast convergence and strong robustness. Therefore, O-ViT imposes orthogonal constraints on the query, key, and value weight matrices in the traditional self-attention, alleviating gradient vanishing problems and maintaining the input feature space. 

Optimizing on the  manifold \cite{smith1994optimization_fic} has achieved impressive performance in CNN and Recurrent Neural Network (RNN) \cite{Wang2020OrthogonalCN,2015Unitary}. For instance, \cite{Huang2017ARN,huang2018building} utilize geometry constraints to construct analogous-convolution architecture as to CNN, and Arjovsky et al. \cite{2015Unitary} uses the norm-stable property of orthogonal matrices to alleviate the gradient explosion and vanishing problem. Moreover, orthogonal initializations of parameters yield depth-independent learning times \cite{Saxe2014ExactST}. However, there is no work to conduct orthogonal optimization in ViT, and this paper is the first attempt to bridge the gap between ViT and geometry optimization.

The gradient backpropagation tends to become difficult with the deep learning optimization problem constrained by the manifold structure \cite{smith1994optimization_fic}. Updating trainable parameters along the manifold involves extensive orthogonal projection and retraction operation calculation \cite{Bronstein2021GeometricDL}. ExpRNN \cite{Casado2019CheapOC} creatively adopts a surjective exponential mapping on the Lie group to achieve cheap optimization and orthogonal constraints. Inspired by it, we pay attention to a surjective mapping between the
orthogonal group and its Lie algebra, allowing O-ViT to transform the geometry optimization into the general optimization problem in Euclidean space. 

Another way to achieve cheap optimization is to substitute hard orthogonal constraints for optimizing on the  manifold, which has been seen in orthogonal CNN and RNN \cite{Wang2020OrthogonalCN}. They use the discrepancy between the identity matrix $E$ and the product of parameter $W$ and its transpose, $WW^T - E$, as a penalty term of the main task.  However, the search space of the primary optimization objective does not necessarily intersect with the hard orthogonality constraints. As a result, hard orthogonal constraints may fail to converge to an optimal point that satisfies both main task and orthogonal constraints,  and the proposed  {O-ViT} avoids this insufficient parameterization. This paper makes the following three major contributions:
\begin{enumerate}
    \item %We observed the feature space distortion problem in ViT. 
    We propose a novel method  named O-ViT  to restrict the self-attention space-projection parameters on the orthogonal manifold, which is the first to  improve ViT in a geometric optimization way.
    %
    %\item O-ViT uses no hard orthogonality constraint. 
    \item O-ViT can pull the geometric optimization back to the Euclidean optimization. Therefore,  O-ViT can avoid complex orthogonal projection and retraction. As a result, it can be optimized by general gradient descent optimizers. %still work, and complex orthogonal projection and retraction can be avoided. 
    Moreover, O-ViT uses no hard orthogonality constraint.
    \item  We conduct comparative experiments between O-ViT and ViT on well-known datasets, which demonstrate the superiority of  {O-ViT}   over other existing ViTs.
\end{enumerate}

This paper is organized as follows:
Section~\ref{sec:02} introduces the related work.
 Section~\ref{sec:03} details the framework and the orthogonal parameterization of our  {O-ViT}.
Section~\ref{sec:04} conducts comparative experiments to investigate O-ViT's superiority over state-of-the-art ViT models.
Section~\ref{sec:05} concludes this paper and discusses  the future work. 
\section{Background}
The proposed  {O-ViT} combines ViT and the optimization on manifold for deep learning.% therefore, we first recall technical details about them.
\label{sec:02}
\subsection{Vision Transformer (ViT)}  
\label{sec:2.1}
Based on the assumption of translation invariance \cite{touvron2021going,DBLP:conf/cvpr/KayhanG20}, CNN shares and translates one convolution kernel filter to extract local features at different positions in one channel. Getting rid of CNN, ViT takes advantage of the self-attention mechanism \cite{DBLP:conf/nips/VaswaniSPUJGKP17} rather than the above assumption. The essence of self-attention is represented as
\begin{equation}
\begin{aligned}
\label{single-head-attention}
   & Q,K,V = XW_Q, XW_K, XW_V,\\
   & Attention(Q,K,V) = softmax(\frac{QK^T}{\sqrt{d_k}})V, \\
\end{aligned}
\end{equation}
where $X$ is the input feature. $W_Q$, $W_K$, and $W_V$ are trainable linear transformation matrices, which are applied to $X$ to generate the query matrix $Q$, the key matrix $K$, and the value matrix $V$, respectively. $d_k$ is the dimension of $K$.  

ViT first embeds the input image into fixed-size patches and then embeds their positional information named PatchEmbedding and Positional Embedding \cite{fayyaz2021ats}.
Then the scaled dot-product self-attention mechanism, which is served as an encoder, is applied to the embedding for feature extraction. 
The self-attention block measures the correlation between different projection spaces, $Q$ and $K$ %by the inner product
, and the normalized correlation is applied to $V$ as the attention map.  
Equation~(\ref{single-head-attention}) is also called single-head self-attention and can be improved by the multi-head self-attention \cite{yan2019tener}:
\begin{equation}
\begin{aligned}
& head^{(h)} = Attention(Q^{(h)},K^{(h)},V^{(h)}) =\\ 
& softmax(\frac{XW_Q^{(h)}{(XW_K^{(h)})}^T}{\sqrt{d_k}})XW_V^{(h)},\\
&MultiHead(X) = [head^{(1)};\cdots;head^{(n)}]W_O,
\end{aligned}
\end{equation}
where $n$ is the number of heads, $h$ is the head index, and $[head^{(1)};\cdots;head^{(n)}]$ means concatenating all single heads in the last dimension. Let $d = n \times d_k$, $W_O$ is a learnable parameter of size $\mathbb{R}^{d \times d}$. Furthermore, ViT can be interpreted by a biologically plausible memory model named Sparse Distributed Memory (SDM) \cite{Bricken2021AttentionAS}.  The intersection of hyperspheres adopted by read operations in SDM can approximate the softmax function in ViT. Although it has nothing to do with orthogonality, such geometry interpretation inspires us to rethink ViT from a geometric perspective.

%It is generally acknowledged that the self-attention mechanism enables ViT to gain overall features even in the shallow layer. 

\subsection{Optimization on Manifold for Deep Learning}
Denoting $\mathcal{D}$ as the admissible search space of the parameter $\theta$, most deep learning methods can be abstracted as the following optimization goal:
\begin{equation}
\label{eq:1}
  \mathop{\arg\min}_{\theta \in \mathcal{D}} \ f_{\theta}(x).
\end{equation}
Deep learning optimization is often non-constraint, or frankly, the solution space $\mathcal{D}$ is defined as the Euclidean space. In order to exploit the underlying geometry structure of solutions, optimization problems have developed to be solved on Riemannian manifolds: 
%which transforms Equation~\ref{eq:1} into
\begin{equation}
\label{eq:2}
  \mathop{\arg\min}_{\theta \in \mathcal{M}} \ f_{\theta}(x).
\end{equation}
where $\mathcal{M}$ denotes the manifold.
Equation~\ref{eq:2} is named Optimization on Manifolds, or Geometric Optimization\cite{smith1994optimization_fic}.
Furthermore, research on geometric optimization in quantum chemistry has also sprung up named Quantum Geometry Optimization (QGOpt), and there also exists a library integrated for QGOpt \footnote{Available from https://qgopt.readthedocs.io/en/latest/}.

A manifold is a collection of objects subject to certain constraints. For instance, the \textit{Stiefel} manifold $St(p, n)$ is a set of matrices $W \in \mathbb{R}^{n \times p} (p \le n)$,
all of which are endowed with the Frobenius inner product \cite{2018Geometry} and subject to $W^TW = I_p$, where $I_p$ denotes the identity matrix of size $\mathbb{R}^{p \times p}$.
Therefore, the optimal solution of geometric optimization should satisfy the constraints determined by specific manifolds. To achieve this goal, there are two steps in the optimization: orthogonal projection and retraction operation \cite{Bronstein2021GeometricDL} (please refer to Figure~\ref{fig:fig1}).

As shown in Figure~\ref{fig:fig1}), there are two nearby points $\theta$ and $\theta^{\prime}$ on a manifold $\mathcal{M}$ together with the tangent space denoted by the blue area. $\theta^{\prime}$ is the next point of $\theta$ in the minimization of objective $f(x;\theta)$.
The tangent space, $T_{\theta} \mathcal{M}$, is a real vector space consisting of all tangent vectors passing through $\theta$.
%The $grad\ f(\theta)$ denotes the Riemannian gradient at the point $\theta$. 
Figure~\ref{fig:fig1}) shows that $f(x;\theta)$ descents steepest in the direction of $\mathbf{H}$, which is the negative direction of the Riemannian gradient $grad\ f(\theta)$  \cite{hawe2013learning_thesis}.
$grad\ f(\theta)$ is a tangent vector on the tangent space $T_{\theta} \mathcal{M}$. It can be obtained by the orthogonal projection $\Pi$, which projects the gradient at a point $\theta$ from the ambient Euclidean space to the tangent space $T_{\theta} \mathcal{M}$ and can be represented as:
\begin{equation}
    grad\ f(\theta) = \Pi _{T_\theta \mathcal{M}}(\nabla f(\theta)),
\end{equation}
where $\nabla f(\theta)$ represents the Euclidean gradient.

A geodesic is a locally shortest path between two points on the manifold  \cite{hawe2013learning_thesis} and can be uniquely determined by a tangent vector.
The smooth red curve in Figure~\ref{fig:fig1}) denotes a geodesic $\Gamma_{\theta}(\gamma \mathbf{H})$ in the direction of $\mathbf{H}$ with a step size $\gamma$. 
The geometric optimization requires to 
update the point $\theta$ to the point  $\theta^{\prime}$ in a search direction 
$\mathbf{H} \in T_{\theta} \mathcal{M}$ along the curve $\Gamma_{\theta}(\gamma \mathbf{H})$.
Due to the high complexity, the geodesic is approximated by the retraction $\mathfrak{R}_{\theta}(\gamma \mathbf{H}) : T_{\theta}\mathcal{M} \rightarrow \mathcal{M}$ in practice \cite{2018Geometry}, which can map updated parameters from the tangent space back to the manifold.

%It points to the direction where $f_{\theta}(x)$ ascends steepest on the manifold  \cite{hawe2013learning_thesis}. Hence, moving the parameter $\theta$ along the Riemannian gradient can gradually reach the optimum. Furthermore, the retraction operation maps the updated parameter  from the tangent space back to the manifold.
\begin{figure}[h]
\centering
\includegraphics[width=3.2in]
{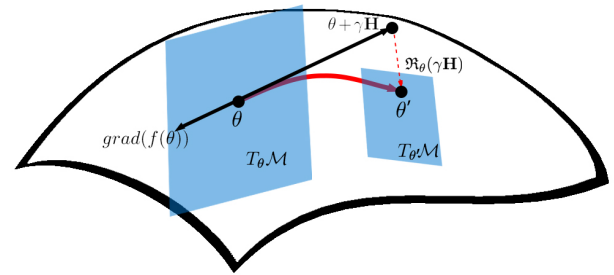}
\caption{Update process in geometric optimization}
\label{fig:fig1}
\end{figure}

%%%%%%%%%%%%%%%%%%%%%%%%%%%%%%%%%%%%%%%%%%%%%%%%%%%%%%%%%%%%%%%%%%%%%%%%%%%%%%%%% 
\section{Our Proposed Method}
\label{sec:03}
We pay attention to the invariant metric inherited by orthogonal matrices and creatively restrict the linear transformation matrices of self-attention in the BaseViT to reside on the orthogonal manifold. We also explore a computationally economic way to parameterize them. 
We first briefly describe O-ViT's architecture in Section~\ref{3.1}. Then we introduce O-ViT's orthogonality technique in Section~\ref{3.3}
and explain theoretical advantages that support its efficiency in optimization in Section~\ref{3.4}.

%%%%%%%%%%%%%%%%%%%%%%%%%%%%%%%%%%%%%
\subsection{O-ViT Architecture}
\label{3.1}
O-ViT architecture differs from other ViTs in the design of the self-attention block. Given the input $X$, O-ViT defines an orthogonal self-attention block as
\begin{equation}
\begin{aligned}
    & Q, K, V = X\, h(A^Q), X\, h(A^K), X\, h(A^V),\\
  \end{aligned}
\end{equation}
where $A^Q$, $A^K$ and $A^V$ are all skew-symmetric matrices   \cite{lee2005geometric}, and they can be extended to skew-Hermitian matrices in case of unitary constraints. 
Algorithm~\ref{alg:alg1} and Algorithm~\ref{alg:alg2} present the orthogonalization performed over the self-attention block. 
As seen in the Algorithm~\ref{alg:alg1}, the query, key, and value weight matrices are imposed orthogonal parameterization (refer to line 1 $\sim$ line 3) before projecting input features $X$ to corresponding query, key and value spaces (refer to line 4 $\sim$ line 6).
As the Algorithm~\ref{alg:alg2} shows, the orthogonal parameterization of O-ViT adopts a two-step strategy. Firstly, line 1 transforms an arbitrary weight matrix to a skew-symmetric one. Then line 2 employs $h(X) = 2(E+X)^{-1} - E$ to map it to the orthogonal group. 
The reason why we use skew-symmetric matrices as the transition to realize orthogonal constraints is that there is a special relationship between them, and it will be detailed in the following subsection. 

\begin{algorithm}[h]
\caption{Orthogonal Self-Attention}
\label{alg:alg1}
\textbf{Input}: $X$\\
\textbf{Parameter}: $W_Q,W_K,W_V$\\
\textbf{Output}: $Q,K,V$
\begin{algorithmic}[1]
%[1] enables line numbers
\STATE $W_Q$ = Parameterization($W_Q$).\\
\STATE $W_K$ = Parameterization($W_K$).\\
\STATE $W_V$ = Parameterization($W_V$).\\
\STATE $Q = XW_Q$.\\
\STATE $K = XW_K$.\\
\STATE $V = XW_V$.\\
\STATE \textbf{return} $Q,K,V$
\end{algorithmic}
\end{algorithm}
\begin{algorithm}[h]
\caption{Parameterization}
\label{alg:alg2}
\textbf{Input}: $W$\\
\textbf{Output}: $W$
\begin{algorithmic}[1]
%[1] enables line numbers
\STATE $W = W - W^T$ \ //skew\_symmetrize %Skew\_Symmetrize($X$).\\
\STATE $W =  2(E+W)^{-1} - E$ \ //orthogonalize %Orthogonalize($X$).\\
\STATE \textbf{return} $W$
\end{algorithmic}
\end{algorithm}

% In Section~\ref{3.2}, we describe good properties of orthogonal transformations, which make O-ViT an advisable choice for taking over traditional linear transformations in the traditional self-attention.
%
% \subsection{Orthogonal Transformations}
% \label{3.2}
The orthogonal manifold is the collection of matrices $A \in \mathbb{R}^{n \times n} $ that satisfies the constraint in Equation~(\ref{orth_definition}).
According to the definition, orthogonal transformations have the following excellent properties that motivate us to replace general linear projections with orthogonal ones in the traditional self-attention. Consequently, orthogonal constraints are enforced on O-ViT.

\textbf{Inner Product Invariance}
Orthogonal transformations do not change the inner product (refer to Appendix~\ref{appendix:1.1}).
An inner product is fundamental for the vector metrics such as length and angles. Therefore, the inner product invariance property can induce \textbf{length} (refer to Appendix~\ref{appendix:1.2}) and \textbf{angle invariance} (refer to Appendix~\ref{appendix:1.3}).

\textbf{Distance Invariance}
The distance between the point $A$ and point $B$ can be represented by the length of a vector $ d(A,B) = \mid \vec{AB} \mid$.
% \begin{equation}
%     d(A,B) = \mid \vec{AB} \mid.
% \end{equation}
Since the orthogonal transformation leaves the vector length unchanged, it is guaranteed to maintain the distance between two points.

Now that the length of vectors, the included angle between vectors, and the distance between points remain the same after the orthogonal transformation, it can keep the geometric structure of the original space undestroyed. O-ViT understands the attention map in ViT from a geometric point of view and is conscious of the excellent geometric property of the orthogonal transformation. Therefore, O-ViT parameterizes the three matrices, which are used to transform the input feature and generate the query, key, and value matrices, as orthogonal matrices instead of the general ones in ViT. As a result, the $Q$, $K$, and $V$ produced by O-ViT can collect and retain features as much as possible. Furthermore, $QK^T$, the correlation between the $Q$ and $K$, is calculated upon a feature map with less distortion, and its confidence level is pushed to a higher degree.

\subsection{Orthogonal Parameterization}
\label{3.3}
O-ViT employs skew-symmetric matrices  \cite{Casado2019CheapOC} as the transition to realize the orthogonal constraint. Hence, we first offer the overview of the orthogonal group \cite{Casado2019CheapOC}, skew-symmetric matrices, and their relationship so that they can be converted to each other.
\begin{definition}
The orthogonal group is formally defined as \cite{Casado2019CheapOC}:
\begin{equation}
    O(n) = \{B \in \mathbb{R}^{n \times n} | B^TB = I\}.
\end{equation}
\end{definition}

\begin{definition}
The unitary group is the extension of the orthogonal group $O(n)$ to the complex domain \cite{Casado2019CheapOC}:
\begin{equation}
    U(n) = \{B \in \mathbb{C}^{n \times n} | B^{\ast}B = I\}.
\end{equation}
\end{definition}

The tangent space at the identity element of the Lie  group $G$ is called the Lie algebra $ \mathfrak{g}$ of the group. The Lie algebras of the special orthogonal group and the unitary group are \cite{Casado2019CheapOC}
\begin{equation}
    \begin{aligned}
    \mathfrak{so}(n) &= \{A \in \mathbb{R}^{n \times n} | A\ + A^T = 0\}, \\
    \mathfrak{u}(n) &= \{A \in \mathbb{C}^{n \times n} | A\ + A^{\ast} = 0\},
    \end{aligned}
\end{equation}
which are known as skew-symmetric and skew-Hermitian matrices, respectively.
They are isomorphic to a vector space \cite{Casado2019CheapOC}. 
\begin{proposition}
Any real square matrix $A \in \mathbb{R}^{n \times n}$ can be mapped into a skew-symmetric matrix by $A - A^T$.
\vspace{-0.1in}
\end{proposition}

\begin{proof}
See Appendix~\ref{appendix:op}.
\end{proof}
By the same token, any complex square matrix $A \in \mathbb{C}^{n \times n}$ can be transformed into a skew-Hermitian matrix.

In the Lie group theory, the exponential mapping $exp: \mathfrak{g} \rightarrow G$  \cite{Casado2019CheapOC} builds correspondence between $\mathfrak{so}(n)$ and its Lie Group $O(n)$. However, the mapping $exp(X) = \Sigma_{n=0}\frac{X^n}{n!}$ is computationally expensive, and the huge number produced by the exponent may induce the gradient vanishing problem in the softmax function. Hence, we use the  map $h: \mathfrak{g} \rightarrow G$  \cite{1953On} to replace it.
 \begin{equation}
   h(X) = 2(E + X)^{-1} - E.
 \end{equation}

\begin{proposition}
The equation $h(X) = 2(E + X)^{-1} - E$ can project any skew-symmetric matrix $X \in \mathbb{R}^{n \times n}$ to the orthogonal group.
\vspace{-0.1in}
\end{proposition}
\begin{proof}
See Appendix~\ref{appendix:op}.
\end{proof}
\begin{theorem}
The equation $h(X) = 2(E + X)^{-1} - E$ is a surjective mapping between the orthogonal group and its Lie algebra \cite{1953On}. For any $Y \in  O(n)$, there exists an  skew-symmetric matrix $X$ that satisfies $ h(X) = Y$.
\vspace{-0.1in}
\end{theorem}

\begin{proof}
See Appendix~\ref{appendix:op}.
\end{proof}

\subsection{From Riemannian to Euclidean Optimization}
\label{3.4}
Manifold optimization belongs to the domain of constrained optimization \cite{DBLP:conf/ijcai/KotaryFHW21}, and the parameters that minimize the optimization objective should satisfy the constraint of Riemannian manifolds in the meantime. In other words, the optimal solution must be searched on the corresponding Riemannian manifold rather than the Euclidean space. 
%Therefore, after being constrained on the orthogonal manifold, the $W_Q$,$W_K$,$W_V$ in O-ViT have to be updated iteratively in the gradient descent direction along the manifold. 
As introduced in Section~\ref{sec:02}, the gradient calculated in the ambient Euclidean space must be projected to the tangent space. 
%It implies that traditional gradient descent optimizers like ADAM or ADAGRAD no longer work.
Kumar et al. \cite{2018Geometry} introduced constraint  Stochastic Gradient Descent-Momentum (SGD-M) and constraint Root Mean Square Prop (RMSProp) as a counterpart of the regular ones in Euclidean space. Nevertheless, the orthogonal projection and the retraction operation are computationally expensive in geometry optimizers. Our {O-ViT}'s parameterization has such excellent properties that it is a sensible option for geometry optimization.

\textbf{Property 1 :} The optimization of  {O-ViT}  can be transformed into an optimization problem in Euclidean space. 
Let $\theta_B$ represent the trainable parameter subjected to the orthogonal group, the constrained optimization problem
\begin{equation}
\label{opt1}
    \min_{\theta_B \in G} f(x; \theta_B),
\end{equation}
is equivalent to following optimization problem 
\begin{equation}
\label{opt2}
    \min_{\theta_A \in \mathfrak{g}} f(x;\theta_A),
\end{equation}
where $\theta_A$ is a skew-symmetric matrix.
Evidently, an optimal solution $\hat{\theta_B}$ for Equation~(\ref{opt1}) and an  optimal solution $\hat{\theta_A}$
for Equation~(\ref{opt2}) have an equivalent relationship that $\hat{\theta_B} = h(\hat{\theta_A})$,
since the map $h: \mathfrak{g} \rightarrow G$ introduced in Section~\ref{3.3} is surjective. Therefore, if the second problem has a solution, then we will definitely find a solution to the first problem.

\textbf{Property 2 :} Our O-ViT does not create saddle points. $h(X) = 2(E + X)^{-1} - E$ constructs a one-to-one correspondence between the skew-symmetric matrices and the orthogonal group. Provided that the optimization problem stays in its tangent space $\mathfrak{o}(n)$, the parameter update is unique. It implies that our parameterization avoids saddle points.

\textbf{Property 3 :} Our O-ViT can be optimized with Euclidean optimizers. Since the skew-symmetric matrix space is isomorphic to a vector space, Equation~(\ref{opt2}) is actually a non-constrained problem. 
As described in Figure~\ref{fig:fig1}, $\Delta \theta_B = -\gamma \ grad\ f(x;\theta_B)$ is on the tangent space $T_{\theta_B}M$, rather than along the geodesic curve \cite{hawe2013learning_thesis} on the manifold. Therefore, $\theta_B$ should be updated by $\theta_B \mathfrak{R}_{\theta}(\Delta \theta_B)$ rather than $\theta_B + \Delta \theta_B$:
\begin{equation}
   \theta_B^{\prime} \leftarrow \theta_B \mathfrak{R}_{\theta}(-\gamma \ grad\ f(x;\theta_B)).
\end{equation}
Moreover, the equation $\hat{\theta_B} = h(\hat{\theta_A})$ induces:
\begin{equation}
   h(\theta_A^{\prime}) \leftarrow h(\theta_A -\gamma  \nabla(f\circ h)(x;\theta_A)),
\end{equation}
where the gradient $\nabla(f\circ h)$ is defined in Euclidean space.
Therefore, trainable parameters of Equation~(\ref{opt2}) are updated in Euclidean space. 
As a consequence, traditional gradient descent optimizers such as ADAM can be directly used to optimize the orthogonality-constrained  {O-ViT}.

\begin{figure*}[ht]
\vskip 0.2in
\centering
% \subfigure[MNIST]{
% \label{mnist3}
% \includegraphics[width=0.46\linewidth,height = 0.16\linewidth]{figure/mnist.pdf}
% }
%
\subfigure[SVHN]{
\label{svhn}
\includegraphics[width=0.49\linewidth]{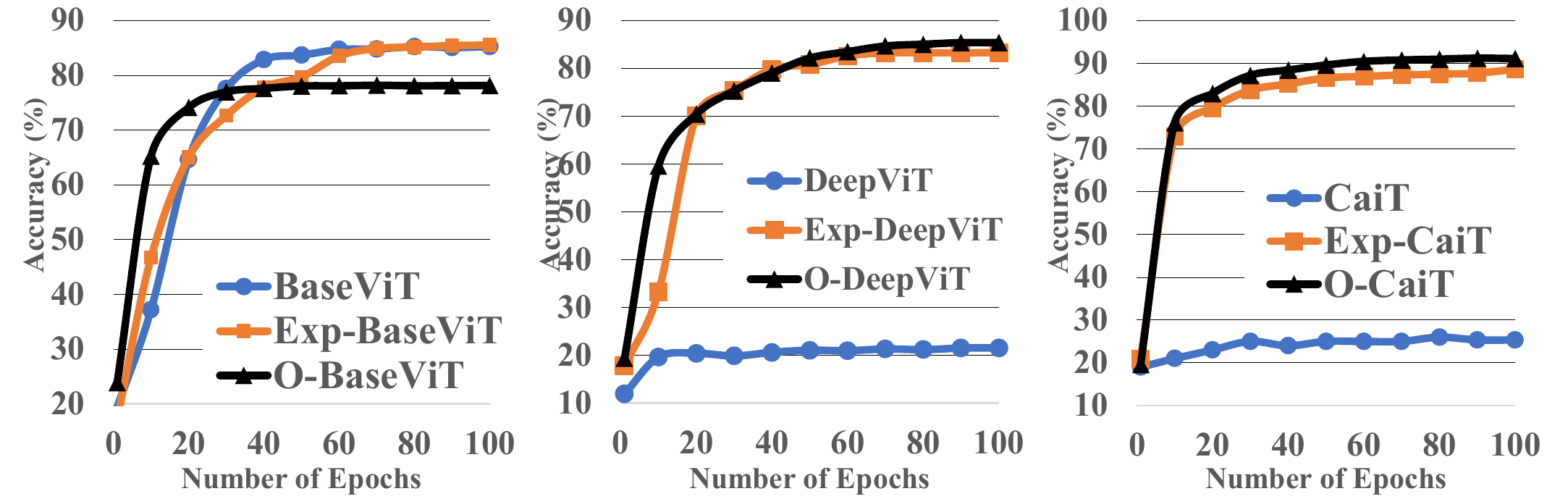}}\hspace{-0.1in}
\subfigure[YALE]{
\label{yale}
\includegraphics[width=0.49\linewidth]{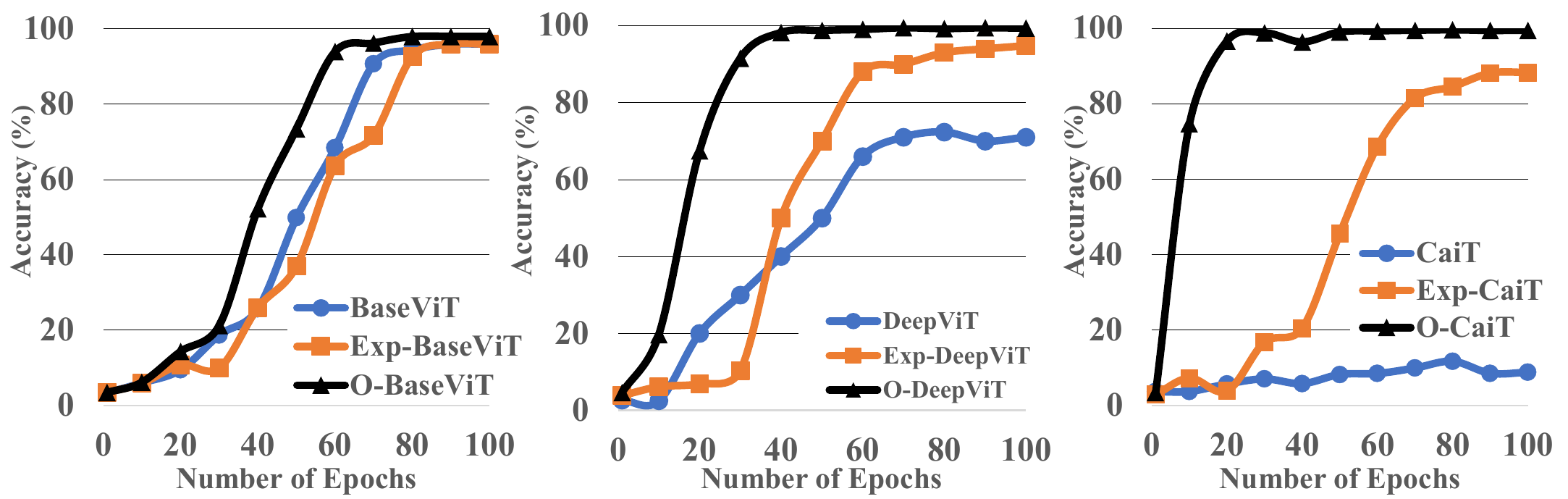}}
\subfigure[CIFAR10]{
\label{cifar10}
\includegraphics[width=0.48\linewidth]{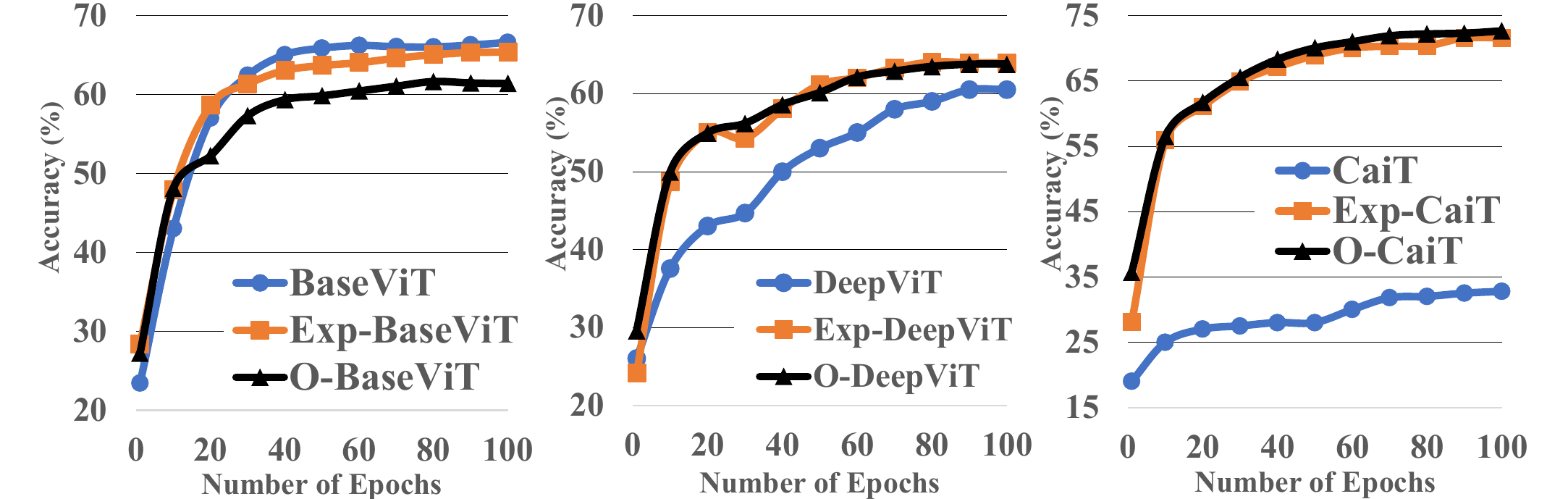}}
\subfigure[CIFAR100]{
\label{cifar100}
\includegraphics[width=0.46\linewidth]{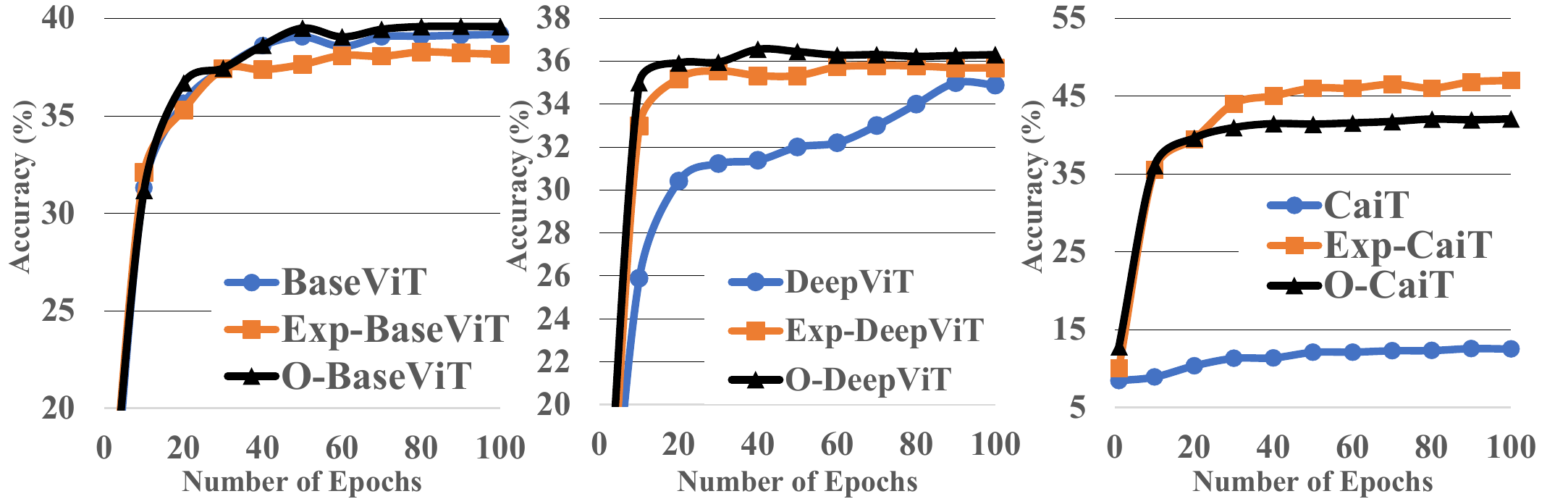}}
\subfigure[Caltech101]{
\label{caltech101}
\includegraphics[width=0.48\linewidth]{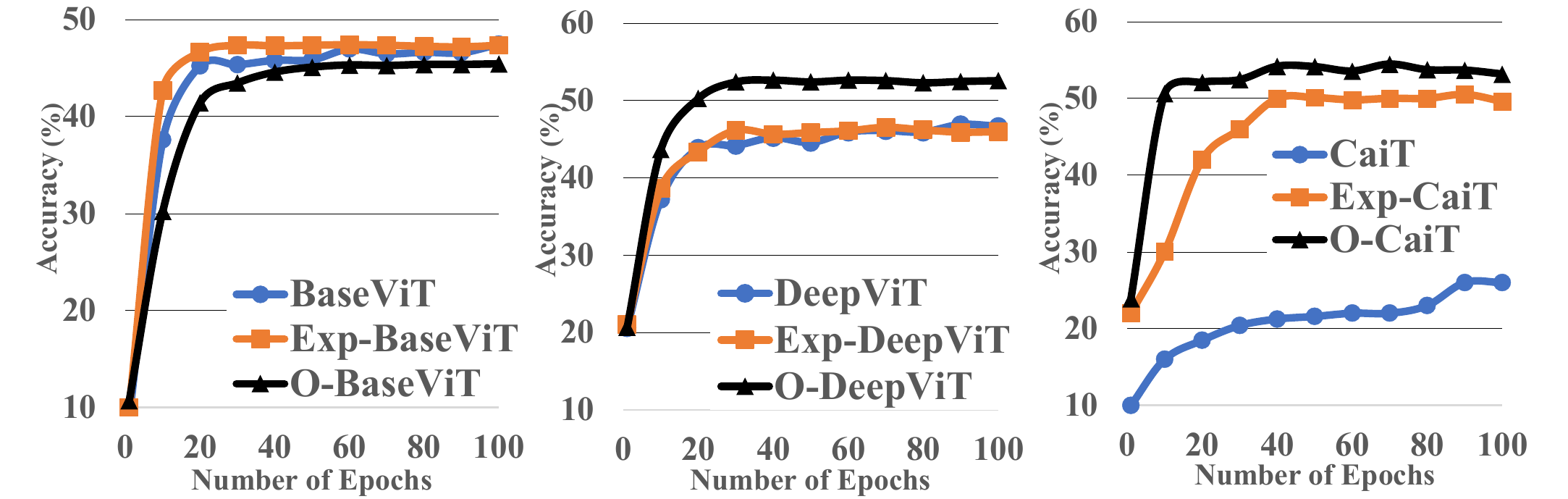}}
\subfigure[ImageNet50]{
\label{ImageNet50}
\includegraphics[width=0.47\linewidth]{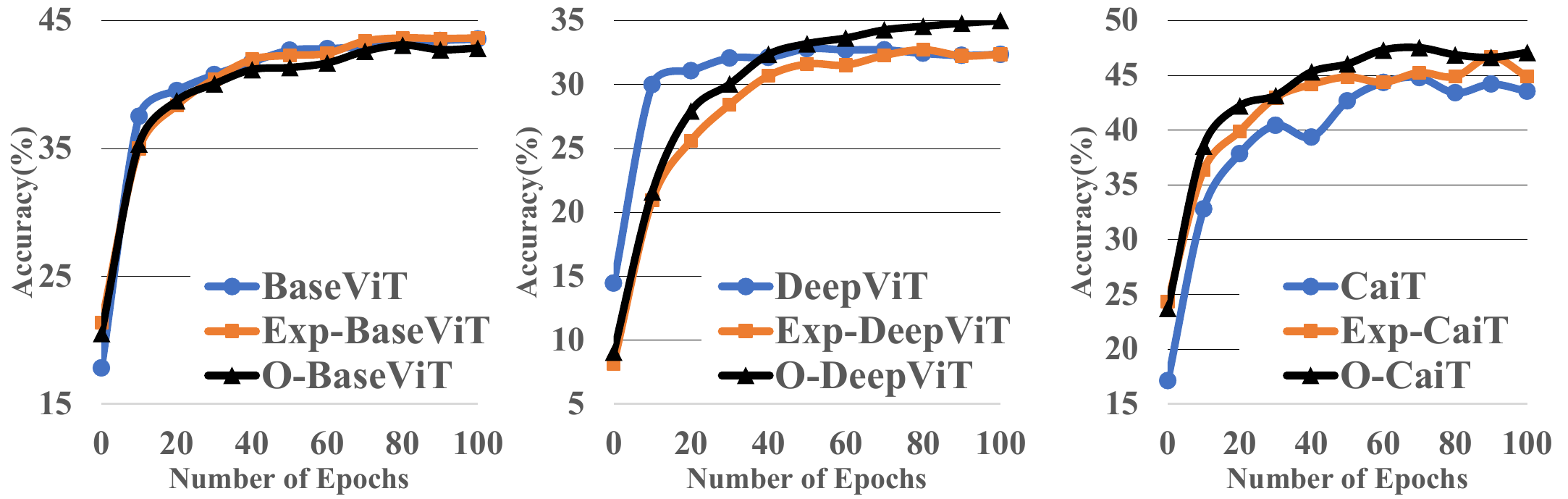}}
\caption{Performance Comparison for Different Datasets}
\label{comparisons}
\end{figure*}

\section{Experimental  Results}
\label{sec:04}
To evaluate the efficiency of our proposed O-ViT, we conducted comparative experiments between  {O-ViT} and ViT on different datasets. We assessed the performance of O-ViT in three aspects: i) under same conditions, the recognition accuracy of  {O-ViT} is higher than ViT  and the convergence of the O-ViT is faster, ii) O-ViT withstands the disturbance of noise better than ViT, and iii) O-ViT can reduce the numbers of parameters while ensuring a credible accuracy.%, and they are discussed in following subsections.

We used following baselines for comparison: i) BaseViT \cite{50650};
ii)  {Exp}onent-orthogonalized  {BaseViT} (\textit{Exp-BaseViT});
iii)  {O}rthogonal  {BaseViT} parameterized by \textit{our} approach (\textit{O-BaseViT}); 
vi) DeepViT  \cite{zhou2021deepvit};
v)  {Exp}onent-orthogonalized  {DeepViT} (\textit{Exp-DeepViT});
vi)  {O}rthogonal  {DeepViT} parameterized by \textit{our} approach (\textit{O-DeepViT});
vii) CaiT \cite{touvron2021going} 
viii)  {Exp}onent-orthogonalized  {CaiT} (\textit{Exp-CaiT});
 and ix)  {O}rthogonal  {CaiT} parameterized by \textit{our} approach (\textit{O-CaiT}).

We used three benchmarks: BaseViT, DeepViT, and CaiT. 
BaseViT means the most original and fundamental ViT. 
DeepViT and CaiT are variant ViTs, and we used them as a supplement to the orginal ViT.
ExpRNN \cite{lezcano2019trivializations} proposed the exponential mapping (Exp) for orthogonal RNNs, and we are the first to use it on ViT. %for optimization on the orthogonal manifold.  
The  {Exp-BaseViT},  {Exp-DeepViT}, and  {Exp-CaiT} methods can be collectively referred to as  {Exp-ViT}s. 
All the approaches beginning with  ``{Exp-}'' and  ``{O-}'' belong to O-ViT.
%, which only realize orthogonal parameterization in a different way.

%By using Python, 
We implemented our O-ViT on top of the deep learning framework PyTorch. Unless otherwise stated, the reported results were measured in Top-1 Accuracy, and we did not take the Top-5 Accuracy into consideration. We set the same cropped size $32 \times 32$ for input data (except for the ImageNet dataset cropped into a size of $224 \times 224$) and the same hyper-parameters for the neural network for the fair comparison in one control group.  We employed a standard data augmentation strategy: random rotation, crop, and horizontal flip. We used SGD as the optimizer. 
We set the learning rate, the weight decay, the the momentum, as $5.0 \times 10^{-3}$, $7.0 \times 10^{−4}$, $0.9$, respectively.
% The learning rate was $5.0 \times 10^{-3}$, weight decay was $7.0 \times 10^{−4}$, and the momentum was set $0.9$. 
We performed experiments on PCs with a single Nvidia GTX 3090 GPU.

\subsection{Ablation}
\label{4.1}
We chose various image recognition tasks to evaluate the performance of O-ViT in comparison with ViT. All results are obtained by training for $100$ epochs from scratch. For each image recognition problem, we want to figure out two issues in the ablation stage: i) what is the efficiency of O-ViT compared with ViT, and ii) whether our orthogonal parameterization worked better than other methods of realizing orthogonal constraints.
Moreover, we selected the DeepViT and CaiT as reference points to show the efficiency of O-ViT on deepening the network architecture, since the DeepViT and CaiT involve more than one attention block. Table~\ref{tab1} presents details of ViT benchmarks. We set the patch size of models to be $16 \times 16$ for the ImageNet dataset and $4\times 4$ for other datasets.

\begin{table}[h]
\begin{center}
\begin{small}
\begin{sc}
\caption{ViT Benchmark Configurations}
\label{tab1}
% \begin{tabular}{\hsize}{@{}@{\extracolsep{\fill}}|l|l|l|@{}}
\begin{tabular}{| c |c |c| c|}
\hline
{\bf Parameters}  & {\bf BaseViT} &{\bf DeepViT}  & {\bf  CaiT}\\
\hline
\hline
{\makecell{\bf Self-Attention  Block \#}}&1 &6 &9 \\
\hline
{\bf Hidden size} & 512 & 512 & 512\\
\hline
{\bf MLP size} & 2048 & 2048 & 2048\\
\hline
{\bf Heads} & 12 & 8 & 8\\
\hline
\end{tabular}
\end{sc}
\end{small}
\end{center}
\end{table}

\subsubsection{Character Recognition Problem}
We chose 
%the MNIST dataset, a subset of the handprinted sample dataset NIST, and 
the Street View House Number (SVHN) dataset to compare the character recognition accuracy between  {O-ViT}s and non-geometric-parameterized ViTs.
%The MNIST dataset is divided into $10$ categories, and each category contains $6,000$ training examples and $1,000$ testing examples. Moreover, it is extended to three channels to adapt to the model architecture. 
The SVHN dataset is a collection of $73257$ training samples and $26032$ testing samples.
They cover a total of $10$ classes.

%Figure~\ref{mnist3} and 
Figure~\ref{svhn} plots the classification accuracy vs. epoch for different methods on the SVHN dataset. 
As seen from it, almost all  {O-ViT}s outperform ViTs in terms of classification accuracy, except for the {O-BaseViT}.
It is not much to say that the  {O-DeepViT} and the  {O-CaiT} have significantly surpassed their opponent methods.
%
%As to MNIST, the  {O-CaiT} is nearly 20\% more accurate than the CaiT, and converges much faster.
%
To be more specific, the  {O-DeepViT} and  {O-CaiT} recognize almost 64\% and 66\% more accurately than the DeepViT and CaiT. 
After applying the orthogonal constraint, all of the three methods go to a higher level regarding recognition precision and efficiency. This fully illustrates that taking orthogonal transformations in the self-attention block, rather than general linear transformations, can improve the visual performance of ViT. 
Moreover, we also pay attention to the influence of orthogonal constraints on the depth of ViT.
Since both DeepViT and CaiT are explorations of deepening ViT, the success of imposing orthogonal constraints on them shows the potential of orthogonal parameterization on increasing the depth of the network.

\begin{table*}[h]
\begin{center}
\begin{small}
\begin{sc}
\caption{Top1-Accuracy Comparison Results of Recognition Problems}
\setlength{\tabcolsep}{3mm}
\label{tab2}
% \begin{tabular}{\hsize}{@{}@{\extracolsep{\fill}}|l|l|l|@{}}
\begin{tabular}{|c|c|c|c|c|c|c|}
 \hline 
%\toprule
%& {\bf MNIST} 
{\bf Method} & {\bf SVHN}& {\bf YALE}& {\bf CIFAR10} & {\bf CIFAR100}& {\bf Caltech101}& {\bf ImageNet50}\\
%\midrule
 \hline 
 \hline 
{\bf BaseViT} %& 97.70\% 
& 85.39\% &95.94\%
& {\bf 66.63\% }&39.34\% &{\bf 47.46}\% & 43.76\%\\
{\bf Exp-BaseViT} %& 98.33\% 
& {\bf 85.54\% }&97.32\% &65.47\%&38.48\% &47.34\% & {\bf 44.0\%}\\
{\bf O-BaseViT} (Ours)%&{\bf 98.39\%}
& 78.39\% &{\bf 98.13\%} & 61.75\% &{\bf 39.68\%} &45.59\%& 43.16\%\\
\hline\hline
{\bf DeepViT}% & 95.33\% 
& 21.59\% &72.32\% &60.51\% &34.89\% &47.34\% &32.88\%\\
{\bf Exp-DeepViT}% & {\bf 99.77\%} 
& 83.33\% &94.82\% &{\bf 64.08\%} &36.43\%&46.78\% & 32.72\%\\
{\bf O-DeepViT} (Ours) %&  97.88 
& {\bf 85.43\%} &{\bf 99.57\%} &63.81\% &{\bf 36.63\%}&{\bf 52.82\%}&{\bf 35.12}\%\\
\hline	\hline
{\bf CaiT}% & 78.53\% 
& 25.35\% &12.01\% &32.77\%&12.01\%&25.99\% &44.80\%\\
{\bf Exp-CaiT}% & 68.47\% 
& 88.63\% & 88.31\% &71.67\%&{\bf46.99}\% &50.51\% &46.68\% \\
{\bf O-CaiT} (Ours)%&{\bf 98.38\%} 
&{\bf 91.31\%}  &{\bf 99.57}\%&{\bf 72.99\%} &42.25\%&{\bf 54.80\%}&{\bf 48.08\%}\\
 \hline
%\bottomrule
\end{tabular}
\end{sc}
\end{small}
\end{center}
\end{table*}

From  %Figure~\ref{mnist3} and 
Figure~\ref{svhn}, we can also see that  {O-ViT}s recognize the SVHN dataset more accurately than  {Exp-ViT}s, except for the  {O-BaseViT}. 
The above success clearly shows that our orthogonal parameterization outperforms other parameterizations on the character recognition task.

\subsubsection{Face Recognition Problem}
We chose the extended YALE face dataset to evaluate the face recognition accuracy of  {O-ViT}s compared to state-of-the-art ViTs. The YALE dataset, belonging to $38$ categories, contains $2314$ training images and $1874$ testing images obtained under various postures and illumination conditions.

Figure~\ref{yale} shows the classification accuracy vs. epoch plots of a series of methods on the YALE dataset, with and without orthogonal optimization. 
%Similar to the character recognition problem, 
From this figure we can see that, compared with ViTs,  {O-ViT}s have an advantage in terms of recognition accuracy and convergence speed. By imposing orthogonal parameterization, the accuracy of the BaseViT is improved by 2.2\%. Moreover, the effectiveness of orthogonal transformation on retaining the structure of feature space is reflected incisively and vividly in the  {O-DeepViT} and the  {O-CaiT}. It is clear from Figure~\ref{yale} that the recognition accuracy of the  {O-DeepViT} and the  {O-CaiT} is far higher than their non-orthogonal counterparts. The  {O-DeepViT} recognizes 27\% higher than the DeepViT, and the  {O-CaiT} recognizes about seven times more accurate than the CaiT. The above comparative results show that, in addition to the character recognition task, the orthogonal constraint on space-projection parameters also performs well in classifying face images at a precise level. Furthermore, the excellent performance of orthogonal constraints on the DeepViT and the CaiT shows that orthogonal parameterization can alleviate over-fitting and take full advantage of the deep ViT to identify face images accurately.

From  Figure~\ref{yale}, we can also see that the  {O-BaseViT} recognizes the facial images at a higher accuracy level than the  {Exp-BaseViT}. Moreover, the  {O-DeepViT} and the  {O-CaiT} have clear advantages in recognition accuracy over the  {Exp-DeepViT} and  {Exp-CaiT} by nearly 4.7\% and 11\%, respectively. Besides, the  {O-DeepViT} and the  {O-CaiT} converge faster than the exponent-orthogonalized  {Exp-DeepViT} and  {Exp-CaiT}. The above success shows that, our orthogonal parameterization has a crucial advantage over other parameterizations both in accuracy and convergence speed towards the facial recognition task.

\subsubsection{Object Recognition Problem}
We chose the CIFAR10, CIFAR100, Caltech101 and ImageNet datasets to perform comparative experiments on the object recognition problem between  {O-ViT}s and ViTs. Caltech101 consists of $7280$ pictures of objects for training and $1864$ pictures for testing. They cover $101$ classes, plus one background clutter class. Both CIFAR10 and CIFAR100 datasets contain $50000$ training images and $10000$ testing images while the former is divided into $10$ categories, and the latter covers $100$ categories. ImageNet is such a huge article classification dataset that there are altogether $1000$ categories in it. Each category covers $500$ training samples and $100$ testing samples. We selected first $50$ categories as a subset for experiments named ImageNet50.

\begin{table*}[ht]
\begin{center}
\begin{scriptsize}
%\begin{sc}
\caption{Comparison between  {O-ViT}s and ViTs with Noises}
\label{tab3}
\begin{tabular}{|c|c|c|c|c|c|c|c|c|c|c|}
 \hline 
%\toprule
\multirow{2}{*}{ {\bf Method}}& \multicolumn{2}{c|}{\bf {YALE}}&\multicolumn{2}{c|}{\bf {SVHN}}  
& \multicolumn{2}{c|}{\bf CIFAR10} & \multicolumn{2}{c|}{\bf Caltech101} & \multicolumn{2}{c|}{\bf ImageNet50}\\
    \cline{2-11} 
    
& {\bf std = 0.1}& {\bf std = 0.05}& {\bf std = 0.1}
& {\bf std = 0.05}& {\bf std = 0.1}& {\bf std = 0.05}& {\bf std = 0.1} & {\bf std = 0.05}& {\bf std = 0.1} & {\bf std = 0.05}\\
%\midrule
\hline
\hline 
 {\bf BaseViT}  & 65.15\% & 94.88\% & 28.66\% & 52.97\% 
  & 36.65\% &\bf 56.32\% &\bf 36.95\%  &\bf 45.42\%   & \bf 35.40\% & \bf 37.92\%\\
%\textbf{Exp-BaseViT} & {\bf 76.95\%} &{\bf 96.74\%} &{\bf 30.54\%} &54.00   &{\bf 40.77\% } &{\bf 58.59\%} &{\bf 38.64\% }&{\bf 46.05\%}\\
 {\bf O-BaseViT} & \bf 74.07\% & \bf 96.37\%  & \bf 28.97\%  &{\bf 54.30\%} & \bf 40.29\%  & 54.51\%  & 33.73\%  & 36.84\% &35.0\%  & 37.04\%\\
 \hline \hline
 {\bf DeepViT}& 64.30\% & 70.92\% & 19.61\%  &19.56\% & 55.16\%  &59.67\% & 41.92\%  & 43.56\% & 30.72\% & 31.68\% \\
%\textbf{Exp-DeepViT}  & 91.86\% &94.72\% & 62.80\% & 73.87\%  &{\bf 57.05\%}  &{\bf 61.80\%}& 49.32\% & 49.55\%\\
 {\bf O-DeepViT} & {\bf 96.16\%} &{\bf 99.04\%}& {\bf 74.40\%} 
&{\bf 82.31\%} &\bf 56.11\% & \bf 60.65\% & {\bf 51.41\%}&{\bf 52.09\%} &\bf 34.08\% &\bf 34.72\%\\
\hline	 \hline
{\bf CaiT}&3.63\% & 3.74\% & 25.30\% &25.56\%  & 32.96\%  &32.69\% &25.25\% & 25.42\% &43.0\% &43.28\%\\
%\textbf{Exp-CaiT}  &  98.29\%  &{\bf 99.31\%} &{\bf 74.24\%} &{\bf 85.96\%} &{\bf 62.04\%} &{\bf 69.7\%}&  50.68\% & 51.64\%\\
{\bf O-CaiT} &{\bf 98.72\%} &{\bf 99.31\%} &\bf 71.71\% & \bf 85.70\% & \bf 59.44\%  &\bf 66.75\%&{\bf 53.22\%} &{\bf 54.46\%}&\bf 44.92\% &\bf 46.48\%\\
\hline
%\bottomrule
\end{tabular}
%\end{sc}
\end{scriptsize}
\end{center}
\end{table*}

\begin{table*}
\centering
\caption{Comparison between O-ViT and ViT in Classification Accuracy and the Number of Parameters}
\label{paramsAbla}
\begin{threeparttable}
\small
\begin{tabular}{|c| c| c| c| c|}
\hline 
 \multirow{2}{*}{\bf Dataset} &
    \multicolumn{2}{c|}{\bf ViT} & \multicolumn{2}{c|}{\bf O-ViT} \\
    \cline{2-5} 
    & {\bf accuracy [\%]} & {\bf parameters [M]} 
    & {\bf accuracy [\%]} & {\bf parameters [M]}\\
\hline 
\hline 
\textbf{YALE} & 95.94 & 4.32 & \textbf{99.15} &\textbf{2.44} \\
\hline 
\textbf{SVHN} &  78.38 & 2.43 &\bf 80.7 & 2.43  \\
\hline
\textbf{CIFAR10} &  61.21 &  2.69 & \textbf{61.88} &\textbf{2.43}  \\
\hline 
\textbf{CIFAR100} &  39.34 & 3.26  &\bf 39.68 &  3.26  \\
\hline 
\textbf{Caltech101} &  47.46 &  4.37 & \textbf{47.97 }& \textbf{2.48} \\
\hline	
\textbf{ImageNet50} & \bf 43.76 & 2.36 %4.33
& {43.16 }& {2.36} \\
\hline
\end{tabular}
\begin{tablenotes}
    \item[1] The notation ``M" represents a unit symbol of one million.
\end{tablenotes}
\end{threeparttable}
\end{table*}

Figure~\ref{cifar10}, Figure~\ref{cifar100}, Figure~\ref{caltech101}, and Figure~\ref{ImageNet50} show the classification accuracy vs. epoch plots of different methods on the object recognition problem. As seen from them, the majority (18 out of 24) of {O-ViT}s outperforms ViTs in the classification accuracy. The comparison between the  {O-CaiT} and CaiT thoroughly reflects the benefits of orthogonal constraints on feature extraction. 
As to CIFAR10 and Caltech101,  {O-CaiT}'s recognition accuracy is about twice that of CaiT.
As to CIFAR100,  {O-CaiT}'s recognition accuracy is more than three times that of CaiT. 
{O-DeepViT} recognizes ImageNet50 dataset more accurately than the DeepViT. Moreover, both the two orthogonal parameterizations,  {Exp-CaiT} and  {O-CaiT}, outperform their non-orthogonal counterpart on the ImageNet50 dataset, which confirms that projecting feature space to the orthogonal manifold can improve the efficiency of feature extraction.

As seen from  Figure~\ref{cifar100},  O-BaseViT and {O-DeepViT} outperform Exp-BaseViT and  {Exp-DeepViT} by a narrow margin.
From  Figure~\ref{caltech101}, we can also see that, the  {OCaiT} recognizes object images more precise than the  {Exp-CaiT}.   Moreover, the  {O-CaiT} converges faster than the  {Exp-CaiT}. The above comparative results shows that, in terms of object recognition task, our orthogonal parameterization achieves better performance than other parameterization approaches both in accuracy and convergence speed.

In summary, the above comparative experiments on different recognition tasks confirm the advantage of   {O-ViT} in terms of classification accuracy and convergence speed over  ViT and other orthogonal parameterizations.
Furthermore, the norm-keeping property of orthogonal matrices help   {O-ViT} increase the depth of the network.
The self-attention mechanism adopts the softmax function to normalize the similarity between the query and key, and exponents in softmax induce zero gradients resulted from very large numbers.
When the zero gradient is transmitted to front layers, the shrinking effects will grow exponentially, yielding the gradient vanishing problem. Parameters are updated in accordance with the direction of gradient descent, hence, the vanishing gradient will inevitably restrict ViT to go deeper.
Orthogonal parameterization can alleviate the above gradient vanishing problem due to its norm-keeing property, thus it can help ViT to go deeper. The success of  {O-DeepViT} and  {O-CaiT} confirms the potential of orthogonal parameterization on increasing the depth of the network.

\subsection{Robustness}
\label{4.2}
To evaluate the robustness of  {O-ViT}s compared to ViTs, we added four kinds of noise to different datasets' testing samples. 
All noises obey the Gaussian distribution with the expected value of $0$. Let $std$ represent the standard deviation of the Gaussian distribution, four Gaussian noises were: i) $std = 0.05$, ii) $std = 0.08$, iii) $std = 0.1$ and iv) $std = 1$. We only show two kinds of noise interference (refer to Table~\ref{tab3}) since space is limited. Please see Appendix~\ref{appendix:5} for the full version.
We employed the recognition accuracy of noise-corrupted images to measure the robustness of methods. Consequently, the higher value of accuracy represents the stronger robustness.

Table~\ref{tab3} shows the comparison between robustness performance between ViTs and  {O-ViT}s on YALE, SVHN, CIFAR10, Caltech101, and ImageNet50 datasets with noises at different intensities. 
We can see that methods with  {O-} prefixes outperform their counterparts in most cases (25 out of 30) considering noise. The above results confirm that orthogonal projections can resist the corruption of input images to a certain extent, which makes O-ViT have stronger robustness than ViT. 
For example, for the SVHN dataset, We can see a sharp increase in the robustness performance of Deep-ViT after imposing orthogonal constraints.
Moreover, the  {O-CaiT} shows obvious advantages over its non-orthogonal counterpart on  YALE and SVHN dataset. %, which indicates O-ViT has better generalization ability than ViT. 
Table~\ref{tab3} also presents that, for other datasets with noise corruption, OViTs perform better ViTs at least 1\% and up to two times. 

To sum up, methods applied orthogonal constraints(O-ViTs) yield a higher recognition accuracy in majority cases with noise turbulence, which confirms the robustness of the orthogonal parameterization under noises.

\subsection{The Number of Parameters}
\label{4.3}
Table~\ref{paramsAbla} shows the recognition accuracy and the number of trainable parameters of O-ViT and ViT on different datasets. As to recognizing the CIFAR10 dataset, O-ViT is more accurate than ViT by a narrow margin while the number of parameters of O-ViT is slightly smaller than that of ViT with the same depth. As to the YALE and Caltech101 dataset, O-ViT recognizes more accurately than ViT while the number of parameters of O-ViT is nearly half of ViT with the same depth. Orthogonal parameters can reduce redundancy theoretically. The above experiment results confirm that O-ViT can reduce the number of parameters and memory consumption while guaranteeing an acceptable accuracy. 

\section{Conclusion}
\label{sec:05}
ViT makes the application of self-attention go further than natural language processing and performs well on image recognition tasks. However, we observe the scale ambiguity problem in ViT and pay attention to the metric invariance property of orthogonal transformations. 
Therefore, we impose orthogonal constraints on ViT and propose a novel approach, O-ViT, to push the boundaries of the existing ViT in a geometric way. 
Moreover, we use an implementation trick based on classic Lie group theory to simplify the constrained optimization over compact Lie groups, in particular $O(n)$ and $U(n)$. It is of independent interest and could have more applications in combination with other machine learning methods.
Furthermore, we have conducted comparative experiments on different vision recognition tasks to provide abundant practical evidence of O-ViT's excellent performance. Experiments also prove the soundness of O-ViT in deepening the self-attention in ViT.

\balance

\bibliographystyle{IEEEbib}
\bibliography{OViT.bib}

%%%%%%%%%%%%%%%%%%%%%%%%%%%%%%%%%%%%%%%%%%%%%%%%%%%%%%%%%%%%%%%%%%%%%%%%%%%%%%%
%%%%%%%%%%%%%%%%%%%%%%%%%%%%%%%%%%%%%%%%%%%%%%%%%%%%%%%%%%%%%%%%%%%%%%%%%%%%%%%
% APPENDIX
%%%%%%%%%%%%%%%%%%%%%%%%%%%%%%%%%%%%%%%%%%%%%%%%%%%%%%%%%%%%%%%%%%%%%%%%%%%%%%%
%%%%%%%%%%%%%%%%%%%%%%%%%%%%%%%%%%%%%%%%%%%%%%%%%%%%%%%%%%%%%%%%%%%%%%%%%%%%%%%
\newpage

%\appendix
%\onecolumn
%\appendixpage
\appendices
%\appendixpage

\section{Properties of Orthogonal Transformations}
\subsection{Inner Product Invariance}
\label{appendix:1.1}
The inner product of vector $x$ and $y$ in the normed space is denoted as $<x,y>$. After applying orthogonal transformation $x^{\prime} = Ax, y^{\prime} = Ay$, the new inner product is
    \begin{equation}
       <Ax,Ay>\ = (Ax)^TAy = x^TA^TAy = x^Ty =\ <x,y>.
    \end{equation}
That is, the inner product remains unchanged after the orthogonal transformation.

\subsection{Length Invariance}
\label{appendix:1.2}
The length of a vector $x$ is denoted as $\mid x \mid$. After imposing the orthogonal transformation, the length becomes:
\begin{equation}
    \mid Ax \mid = \sqrt{(Ax)^TAx} = \sqrt{x^TA^TAx} = \sqrt{x^Tx} =\ \mid x \mid,
\end{equation}
that is, the length of the vector stays the same after the orthogonal transformation.

\subsection{Included Angle Invariance}
\label{appendix:1.3}
The included angle $\theta$ between the vector $x$ and $y$ is called
\begin{equation}
    \theta = \frac{<x,y>}{\mid x \mid\ \mid y \mid}.
\end{equation}
Since the orthogonal transformation keeps the vector inner product and length unchanged,  there is no doubt that it will hold the included angle.

\section{Orthogonal Parameterization}
\label{appendix:op}

\begin{proposition}
\label{appendix:2}
Any real square matrix $A \in \mathbb{R}^{n \times n}$ can be mapped into a skew-symmetric matrix by $A - A^T$.
\end{proposition}
\begin{proof}
\begin{equation}
\begin{aligned}
% & Y = A - A^T \\
% &   Y + Y^T =
(A - A^T) + (A - A^T)^T = (A - A^T) + (A^T - A) = 0
% &= A - A^T + A^T - A = (A-A)+(A^T - A^T) = 0
\end{aligned}
\end{equation}
$\therefore A - A^T$ \ is\ a\ skew-symmetric\ matrix.
\end{proof}

\begin{proposition}
\label{appendix:3}
The equation $h(X) = 2(E + X)^{-1} - E$ can project any skew-symmetric matrix $X \in \mathbb{R}^{n \times n}$ to the orthogonal group.
\end{proposition}
\begin{proof}
\begin{equation}
\begin{aligned}
%& Y = h(X) = 2(E + X)^{-1} - E\\
&h(X)h^T(X) = 
[2(E + X)^{-1} - E][2(E + X)^{-1} - E]^{T}\\
&= [2(E + X)^{-1} - (E + X)^{-1}(E+X)]
\\&[2(E + X)^{-1} - (E + X)^{-1}(E+X)]^{T}\\
 %          = (E + X)^{-1}(E-X) \\
% & Y^T = [(E + X)^T]^{-1}(E-X)^{T} = (E - X)^{-1}(E+X) \\
%& = (E + X)^{-1}(E-X)[(E + X)^{-1} (E-X)]^T\\
&=(E + X)^{-1} (E-X)(E + X)(E-X)^{-1}\\
& = (E + X)^{-1}(E + X)
= E
\end{aligned}
\end{equation}
$\therefore h(X) = 2(E + X)^{-1} - E$ can map a skew-symmetric matrix to the orthogonal group.
\end{proof}

\begin{theorem}
\label{appendix:4}
The equation $h(X) = 2(E + X)^{-1} - E$ is a surjective mapping between the orthogonal group and its Lie algebra \cite{1953On}. For any $Y \in  O(n)$, there exists an skew-symmetric matrix $X$ that satisfies $ h(X) = Y$.
\end{theorem}
\begin{proof}
$\forall Y \in O(n)$, we have $X =  2(E + Y)^{-1} - E$ \cite{1953On} satisfies:
\begin{equation}
\begin{aligned}
&    h(X) = h(2(E + Y)^{-1} - E)
  = 2[E + [2(E + Y)^{-1} - E]]^{-1} - E\\ 
% & = 2[E - E + 2(E + Y)^{-1}]^{-1} - E  
& = 2[2(E + Y)^{-1}]^{-1} - E 
  %= 2[\frac{1}{2}[(E + Y)^{-1}]^{-1}]- E   
 = 2[\frac{1}{2}(E + Y)]- E
 %& = (E + Y) - E 
 = Y\\
\end{aligned}
\end{equation}
and $X$ can be proven to be a skew-symmetric matrix.
\end{proof}

\section{Comparison between  {O-ViT}s and ViTs Considering Noises}
\label{appendix:5}

Table~\ref{tab4}, Table~\ref{tab5} and Table~\ref{tab6} show the robustness comparison  between  {O-ViT}s and ViTs considering four different kinds of noises, respectively.

\begin{table*}[h]
\vskip 0.15in
\begin{center}
\begin{small}
\begin{sc}
\caption{Comparison between  {O-ViT}s and ViTs with Noises}
\label{tab4}
% \begin{tabular}{\hsize}{@{}@{\extracolsep{\fill}}|l|l|l|@{}}
\begin{tabular}{|c|c|c|c|c|c|c|c|c|}
\hline 
%\toprule
\multirow{2}{*}{ {\bf Method}}&\multicolumn{4}{c|}{\bf SVHN} & \multicolumn{4}{c|}{\bf CIFAR10} \\
    \cline{2-9} 
    
& {\bf std = 1} & {\bf std = 0.1}& {\bf std = 0.08}& {\bf std = 0.05}
& {\bf std = 1} & {\bf std = 0.1}& {\bf std = 0.08}& {\bf std = 0.05}\\
%\midrule
 \hline 
 \hline
{\bf BaseViT}& 9.18\% & 28.66\% &36.16\%
& 52.97\% &\bf 10.98\% & 36.65\% &44.36\%
&\bf 56.32\% \\
%\textbf{Exp-BaseViT} &{\bf 11.82\% }&{\bf 30.54\%} &{\bf 37.35\%} &54.00  &{\bf 12.16\%} &{\bf 40.77\% }&{\bf 47.4\%} &{\bf 58.59\%} \\
{\bf O-BaseViT}   &\bf 9.60\% &\bf 28.97\% &\bf 36.60 \% &{\bf 54.30\%} & 10.65\% &\bf 40.29\% &\bf 45.30 \% & 54.51\%  \\
\hline
\hline
{\bf DeepViT}&{\bf 19.50\%} & 19.61\% &19.58\% &19.56\% & 11.09\% &
55.16\% &57.29\% &59.67\% \\
%\textbf{Exp-DeepViT} & 9.43\% & 62.80\% &67.95\% & 73.87\% &{\bf 12.75\%} &{\bf 57.05\%} &{\bf 59.82\%} &{\bf 61.80\%}\\
{\bf O-DeepViT}   & 10.38 \% & {\bf 74.40\%} &{\bf 78.16\%} &{\bf 82.31\%} &\bf 11.11\% &\bf 56.11\% &\bf 58.49\% &\bf  60.65\% \\
\hline	
\hline
{\bf CaiT}&{\bf 18.79\%} & 25.30\% &25.39\% &25.56\% &{\bf 16.41\%} & 32.96\% &32.5\% &32.77\%\\
%\textbf{Exp-CaiT} & 11.46\% &{\bf 74.24\%} &{\bf 79.86\%} &{\bf 85.96\%} & 12.34\% &{\bf 62.04\%} &{\bf 65.97\%} &{\bf 69.7\%}\\
{\bf O-CaiT}   & 12.32\% &\bf 71.71\%  &\bf 78.59\%&\bf 85.70\% & 12.34\% &\bf 59.44\%  &\bf 62.52 \%&\bf 66.75\% \\
 \hline
 \end{tabular}
 \end{sc}
\end{small}
\end{center}
\vskip -0.1in
\end{table*}

\begin{table*}[h]
\centering
\caption{Comparison between  {O-ViT}s and ViTs with Noises}
\label{tab5}
% \begin{tabular}{\hsize}{@{}@{\extracolsep{\fill}}|l|l|l|@{}}
\small
\begin{tabular}{|c|c|c|c|c|c|c|c|c|}
 \hline 
%\toprule
\multirow{2}{*}{ {\bf Method}}&\multicolumn{4}{c|}{\bf CIFAR100} & \multicolumn{4}{c|}{\bf Caltech101} \\
    \cline{2-9} 
    
& {\bf std = 1} & {\bf std = 0.1}& {\bf std = 0.08}& {\bf std = 0.05}
& {\bf std = 1} & {\bf std = 0.1}& {\bf std = 0.08}& {\bf std = 0.05}\\
%\midrule
 \hline 
 \hline
\textbf{BaseViT}&\bf 1.2\% &\bf 11.49\% &\bf 17.15\%
&{\bf 29.13\%} & 1.13\% &\bf 36.95\% &\bf 41.07\%
&\bf 45.42\% \\
%\textbf{Exp-BaseViT} &{\bf 1.25\% }&{\bf 11.59\%} & {\bf 17.44\%} & 28.22\% & 1.53\% &{\bf 38.64\% }&{\bf  42.03\% }&{\bf 46.05\%} \\
\textbf{O-BaseViT}   & 1.06\% & 10.0\% & 14.83\% & 26.96\%  & {\bf 6.78\%} & 33.73\% & 34.92\% & 36.84\%  \\ 
\hline
\hline
\textbf{DeepViT} &{\bf 1.83\%} & 26.72\% & 29.63\% & 33.42\%  & 3.62\% & 41.92\% & 42.54\% & 43.56\% \\
%\textbf{Exp-DeepViT} & 1.22\% & 24.03\% & 27.43\% & 31.73\% & 2.03\% & 49.32\% &48.59\% & 49.55\%\\
\textbf{O-DeepViT}   & 1.24 \% & {\bf 29.15 \%} &{\bf 31.72\%} &{\bf 34.09\%} & {\bf 4.86\%} & {\bf 51.41\%} &{\bf 51.64\%} &{\bf 52.09\%} \\
\hline	
\hline
\textbf{CaiT}&{\bf 2.69\%} &  11.74\% & 11.99\% & 11.77\% &  {\bf 16.78\%} & 
25.25\% & 25.37\% & 25.42\%\\
%\textbf{Exp-CaiT} &  1.81\% &  34.09\% &{\bf 38.92\%} &{\bf 44.65\%} &  2.15\% &  50.68\% &  50.96\% & 51.64\%\\
\textbf{O-CaiT}   & 1.74\% &{\bf 34.42\%}  &\bf 37.86\% &\bf 42.09\% & 4.01\% &{\bf 53.22\%}  &{\bf 54.01}\%&{\bf 54.46\%} \\
\hline
%\bottomrule
\end{tabular}
\end{table*}

\begin{table*}[h]
\vskip 0.15in
\begin{center}

\begin{sc}
\caption{Comparison between  {O-ViT}s and ViTs with Noises}
\label{tab6}
\small
\begin{tabular}{|c|c|c|c|c|c|c|c|c|}
 \hline 
\multirow{2}{*}{ {\bf Method}} & \multicolumn{4}{c|}{\bf YALE}& \multicolumn{4}{c|}{\bf ImageNet50} \\
    \cline{2-9} 
    
& {\bf std = 1} & {\bf std = 0.1}& {\bf std = 0.08}& {\bf std = 0.05}& {\bf std = 1} & {\bf std = 0.1}& {\bf std = 0.08}& {\bf std = 0.05}\\
%\midrule
\hline
 \hline 
 {\bf BaseViT}  & 2.19\% & 65.15\% & 79.94\% & 94.88\% &\bf 3.76\%& \bf 35.40\% &\bf 36.28\% & \bf 37.92\%\\
%\textbf{Exp-BaseViT} & {\bf 58.26}\% & {\bf 99.99}\% & {\bf 100\%} & {\bf 100\%} & 4.11\% & {\bf 76.95\%} &{\bf 88.37\%} &{\bf 96.74\%} \\
{\bf O-BaseViT} & {\bf 4.75\%} &\bf 74.07\% &\bf 85.38\% &\bf 96.37\% &3.72\% &35.0\% & 35.96\% & 37.04\% \\
\hline
\hline
{\bf DeepViT} & 4.16\% & 64.30\% & 67.02\% & 70.92\%&\bf 4.16\% & 30.72\% &31.24\% & 31.68\%\\
%\textbf{Exp-DeepViT} & {\bf 94.85}\% & {\bf 99.69}\% &{\bf 99.71\%} &{\bf 99.72\%} & 5.66\% & 91.86\% &93.38\% &94.72\%\\
{\bf O-DeepViT}  &{\bf 5.76 \%} & {\bf 96.16\%} &{\bf 97.28\%} &{\bf 99.04\%} & 3.72\% &\bf 34.08\% &\bf 35.08\% &\bf 34.72\%\\
\hline	
\hline
{\bf CaiT} & 2.78\% &  3.63\% & 3.68\% & 3.74\%&5.88\% &43.0\% & 43.12\% &43.28\%\\
%\textbf{Exp-CaiT} & {\bf 98.84}\% &  {\bf 99.94\%} &{\bf  99.95\% }& {\bf 99.95\%}  &  8.06\% &  98.29\% &  98.94\% &{\bf 99.31\%}\\
{\bf O-CaiT}\% &{\bf 6.30\%} &{\bf 98.72\%}  &{\bf 99.04}\%&{\bf 99.31\%} &\bf 6.20\% &\bf 44.92\% & \bf 45.60\% &\bf 46.48\%\\
 \hline
%\bottomrule
\end{tabular}
\end{sc}
\end{center}
\vskip -0.1in
\end{table*}

\end{document}